\ifdefined\pdfminorversion\pdfminorversion=7\fi
\documentclass[11pt,a4paper,logo]{lumia}

\usepackage[numbers,sort&compress]{natbib}

\usepackage{algorithm}
\usepackage{algpseudocode}
\usepackage{subcaption}
\usepackage{array}
\usepackage{threeparttable}
\usepackage{multirow}
\usepackage{makecell}
\usepackage{wrapfig}
\usepackage{siunitx}
\usepackage[section]{placeins}
\newtheorem{lemma}{Lemma}
\newtheorem{corollary}{Corollary}
\newtheorem{proposition}{Proposition}
\newtheorem{remark}{Remark}

\usepackage{amsmath,amsfonts,bm}

\def\eqref#1{Equation~\ref{#1}}

\def\1{\bm{1}}

\DeclareMathAlphabet{\mathsfit}{\encodingdefault}{\sfdefault}{m}{sl}
\SetMathAlphabet{\mathsfit}{bold}{\encodingdefault}{\sfdefault}{bx}{n}

\graphicspath{{figures/}}

\providecommand{\bestval}[1]{\textbf{#1}}

\newcommand{\ohrinline}[1]{\,{\scriptsize(#1\%)}}

\definecolor{citecolor}{HTML}{0071bc}
\definecolor{abstrabg}{HTML}{F2D1C5}
\hypersetup{
  colorlinks=true,
  linkcolor=red,
  citecolor=citecolor,
  filecolor=magenta,
  urlcolor=magenta,
  pdftitle={Beyond One-Size-Fits-All: Sample-Adaptive Strategy Routing for Vision Token Pruning in MLLMs},
  pdfauthor={Haiji Liang, Pengfei Zhou, Zhenglin Wan, Yang You, Wei Wang, Wangbo Zhao}
}
\setheadertitle{Beyond One-Size-Fits-All: Sample-Adaptive Strategy Routing for Vision Token Pruning}

\title{Beyond One-Size-Fits-All:\\
Sample-Adaptive Strategy Routing\\
for Vision Token Pruning in MLLMs}

\author{ \parbox{0.94\textwidth}{\centering 
Haiji Liang$^{1,*,\dagger}$
\quad Pengfei Zhou$^{1,2,*,\dagger}$
\quad Zhenglin Wan$^{1}$
\quad Wei Wang$^{3}$
\quad Yang You$^{1,\ddagger}$
\quad Wangbo Zhao$^{3,\ddagger}$\\ 
$^{1}$National University of Singapore \quad
$^{2}$InfRec, Cardinal AI Lab \quad \\
$^{3}$The Hong Kong University of Science and Technology\\ 
\footnotesize $^{*}$Equal contribution \quad $^{\dagger}$ Project leads \quad $^{\ddagger}$Corresponding authors}}

\begin{document}

\begin{abstract}
Multimodal large language models (MLLMs) process hundreds or thousands of visual tokens per image, incurring prohibitive inference costs.
While existing vision token pruning methods mitigate this overhead, they implicitly assume that a single fixed pruning strategy can be applied uniformly across all inputs. 
Our analysis further reveals that ranking pruning methods by average benchmark accuracy conceals substantial sample-wise complementarity: although the average-best strategy excels overall, alternative strategies prove superior on a significant fraction of individual samples. 
To harness this diversity, we propose \textbf{VIP-Router}, a lightweight \textbf{VI}sion \textbf{P}runing \textbf{Router} that adaptively selects the pruning strategy predicted to be best suited to each input at a specified pruning level. 
Conditioned on low-cost visual and textual features, VIP-Router identifies the most suitable candidate strategy while retaining full-token inference as an option when pruning is predicted to be unfavorable.
Evaluated on a curated suite of pruning-sensitive visual perception benchmarks, VTC-Bench Group A, VIP-Router consistently outperforms the best fixed strategy baseline across all reduction ratios, achieving a 26.9\% relative improvement in average accuracy, and a 22.0\% relative increase in average utility after accounting for realized token cost.
Crucially, VIP-Router operates in a plug-and-play manner without modifying underlying pruning algorithms or model weights, introducing trainable parameters equivalent to merely 0.017\% of the backbone.
Furthermore, VIP-Router proves effective across various MLLM backbones and yields consistent gains on unseen benchmarks, highlighting the potential of sample adaptive routing for visual token pruning.

\end{abstract}

\maketitle

\section{Introduction}

Multimodal large language models (MLLMs) have demonstrated impressive capabilities across a broad spectrum of vision-language tasks~\citep{LLaVA-1.5_CVPR2024,LLaVA-OneVision_TMLR2024,Qwen2VL_2024,Qwen2.5VL_2025,InternVL3_2025}. 
However, the high computational cost of processing visual tokens remains a critical bottleneck for their practical deployment: an image can introduce hundreds or even thousands of tokens into the language model, dominating both memory and latency~\citep{EffiMLLM-Survey_VI2025}. 
To alleviate this overhead, vision token pruning has emerged as a mainstream and effective approach~\citep{FastV_ECCV2024,PruMerge_ICCV2025,VisionZip_CVPR2025,DART_EMNLP2025}. 
These methods leverage diverse signals, including attention scores~\citep{FastV_ECCV2024}, token importance~\citep{VisionZip_CVPR2025}, and semantic similarity~\citep{PruMerge_ICCV2025}, to identify and drop redundant tokens.

Yet, VTC-Bench~\citep{VTC-Bench_ACL2026} shows that benchmarks commonly used to evaluate pruning methods were designed for general perception rather than specifically stress-testing vision token pruning.
On the pruning-sensitive VTC-Bench Group~A (samples answered correctly with full tokens but incorrectly after equivalent-ratio downsampling), we observe that no single pruning strategy dominates across all benchmarks: as shown in Figure~\ref{fig:pruning_complementarity} (a), different strategies excel on different tasks, indicating a non-trivial degree of complementarity.
More critically, beyond this benchmark-level divergence, ranking pruning methods solely by average benchmark performance further obscures the complementarity at the sample level.
Among samples for which at least one pruning strategy succeeds, over one-third are cases where Best Fixed (the best fixed pruning method) fails but an alternative strategy answers correctly (Figure~\ref{fig:pruning_complementarity} (b)).
Consequently, the strategy that performs best on average can still be suboptimal for a substantial fraction of individual inputs.
In short, \textbf{average-best does not imply sample-best.}

\begin{figure}[t] 
\centering 
\includegraphics[width=1\linewidth]{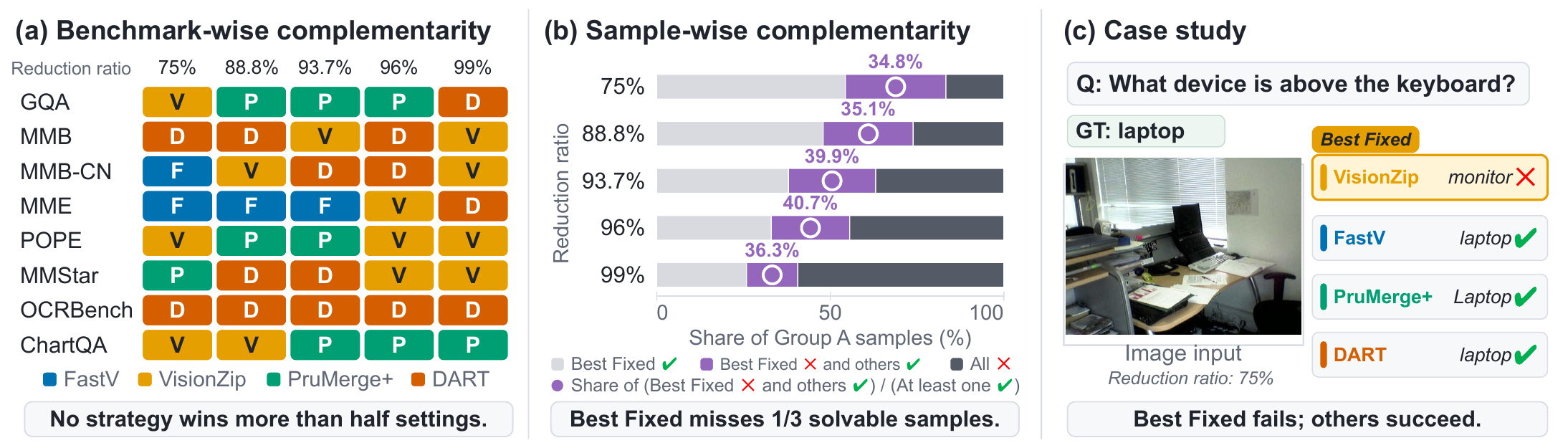}
\caption{(a) The accuracy-best strategy varies across the full VTC-Bench Group A and reduction ratios, with no strategy winning more than half of the 40 benchmark-ratio settings. (b) At each reduction ratio, the stacked bars split Group A samples into those answered correctly by Best Fixed, those missed by Best Fixed but answered correctly by another strategy, and those missed by all strategies. Purple dots report the fraction of samples solvable by at least one strategy that are wrong with Best Fixed, which exceeds one-third at every ratio. (c) In this example, Qwen2-VL yields contrasting predictions across different pruning strategies, illustrating their inherent complementarity.} 
\label{fig:pruning_complementarity} 
\end{figure}

Therefore, rather than designing yet another heuristic to chase higher average performance, we treat this complementarity as an untapped source of improvement and exploit it through sample-adaptive selection.
The primary challenge is that the sample-best choice is defined by the inference outcome, which is inherently unavailable beforehand; identifying it exactly would require executing multiple candidate strategies and comparing their respective outputs, thereby defeating the very computational savings that pruning is intended to provide.
Yet, intuitively, the suitability of a pruning strategy should depend on the visual and linguistic characteristics of the current input.
Our key finding is that low-cost visual and textual features already provide informative signal to predict strategy suitability.

To this end, we propose \textbf{VI}sion \textbf{P}runing \textbf{Router} (\textbf{VIP-Router}), a lightweight, plug-and-play module compatible with existing pruning methods.
VIP-Router constructs a query-conditioned preview representation from visual and textual features using text-to-vision cross-attention, and uses a ratio-conditioned utility predictor to score each candidate option.
At inference time, only the highest-scoring option is executed, with full-token inference included as an explicit candidate when pruning is predicted to be unfavorable.

In our evaluations, VIP-Router consistently outperforms the best fixed pruning strategies across all five reduction ratios on VTC-Bench Group~A.
Accounting for its realized token cost, VIP-Router improves average utility from \(30.88\) to \(37.68\), corresponding to a \(22.0\%\) relative improvement over Best Fixed; meanwhile, average accuracy increases from \(40.35\%\) to \(51.19\%\), a \(26.9\%\) relative gain.
These gains require no modification to either the candidate pruning methods or the underlying MLLM computation pipeline, while VIP-Router adds only about 1.4M trainable parameters, approximately \(0.017\%\) of the MLLM parameters.

Our main contributions are summarized as follows:
\begin{itemize}

\item We characterize the sample-level complementarity among existing vision token pruning strategies on pruning-sensitive benchmarks, revealing substantial per-sample suboptimality masked by standard aggregate protocols.

\item We introduce \textbf{VIP-Router}, a plug-and-play router that uses low-cost features to predict per-strategy utility conditioned on the pruning level while retaining full-token inference as an option when pruning is unfavorable.

\item Extensive experiments establish consistent cost-aware accuracy gains across pruning levels, applicability across different MLLM backbones, and positive zero-shot transfer to unseen benchmarks.

\end{itemize}

\section{Related Work}

\subsection{Vision Token Pruning for MLLMs} 

Vision token pruning is a widely studied approach to improving MLLM efficiency by removing or consolidating redundant visual tokens before or within the language model. Existing methods primarily differ in the criterion used to identify which visual information should be retained. 
Attention-based methods such as FastV and PruMerge exploit attention signals from the language model or visual encoder to identify salient tokens~\citep{FastV_ECCV2024,PruMerge_ICCV2025}, while FitPrune~\citep{FitPrune_AAAI2025} derives pruning configurations by matching attention statistics before and after pruning. Other methods explicitly target redundancy among visual representations: VisionZip~\citep{VisionZip_CVPR2025} preserves dominant tokens while compressing contextual ones, DivPrune~\citep{DivPrune_CVPR2025} promotes diversity among retained tokens, and DART~\citep{DART_EMNLP2025} removes tokens according to their duplication with representative pivots. Recent work has also examined failure modes of pruning; for example, RVIS~\citep{RVIS_2026} adapts pruning during decoding in response to changes in task-relevant visual information. 

Despite these differences, existing approaches fundamentally rely on a single fixed pruning criterion for all inputs. VIP-Router instead treats existing pruning strategies as complementary candidates and performs sample-level selection among them. 

\subsection{Adaptive Inference for Efficient MLLMs} 

Adaptive inference methods adjust MLLM computation according to the input or runtime efficiency requirements rather than using a single fixed inference configuration~\citep{agentrouter_2026}. One line of work adapts the amount or location of token computation. 
Dynamic-LLaVA~\citep{Dynamic-LLaVA_ICLR2025} sparsifies visual and linguistic contexts during prefill and decoding; ATP-LLaVA and SparseVLM~\citep{ATP-LLaVA_CVPR2025,SparseVLM_ICML2025} determine input- or layer-dependent visual-token retention ratios; and AIM~\citep{AIM_ICCV2025} combines pre-LLM merging with progressive in-LLM pruning to support different efficiency requirements. F$^3$A~\citep{F3A_2026} further uses question-conditioned signals to adapt visual-token allocation at a specified pruning level. Another line adapts the visual representation or compression pathway. M$^3$~\citep{M3_ICLR2024} constructs nested representations at multiple visual-token granularities, while MQT~\citep{MQT_NIPS2024} produces different numbers of visual tokens with a shared query transformer. AdaLLaVA~\citep{AdaLLaVA_ICCV2025} operates at a broader structural level by dynamically reconfiguring MLLM computation under runtime latency constraints. Most closely related to VIP-Router, QMoP~\citep{QMoP_2026} uses a query-guided router over pooling-, resampler-, and pruning-based compression branches. However, QMoP jointly learns specialized compression branches and fuses their output representations, potentially with minor performance degradation, whereas VIP-Router leaves existing pruning operators unchanged and makes a discrete selection among off-the-shelf strategies. 

Overall, prior adaptive methods primarily vary the amount or location of computation, the granularity of visual representations, or the compression pathway. VIP-Router operates on a different adaptation axis: which pruning criterion to apply to each sample at a specified pruning level.

\section{Vision Pruning Router}
\label{sec:method}

\subsection{Overview} 
\label{sec:method_overview} 

VIP-Router implements sample-level pruning strategy selection through three components: frozen visual and textual preview encoders, a text-to-vision cross-attention module, and a ratio-conditioned MLP utility predictor.
As illustrated in Figure~\ref{fig:framework}, the lightweight encoders provide global and token-level representations, cross-attention extracts query-conditioned visual information, and the MLP jointly estimates the utility of all candidate strategies at the given pruning ratio.
The following sections (\ref{sec:problem formulation}-\ref{sec:utility_routing}) detail the routing objective, preview representation, and utility predictor.

\begin{figure}[t] 
\centering 
\includegraphics[width=1\linewidth]{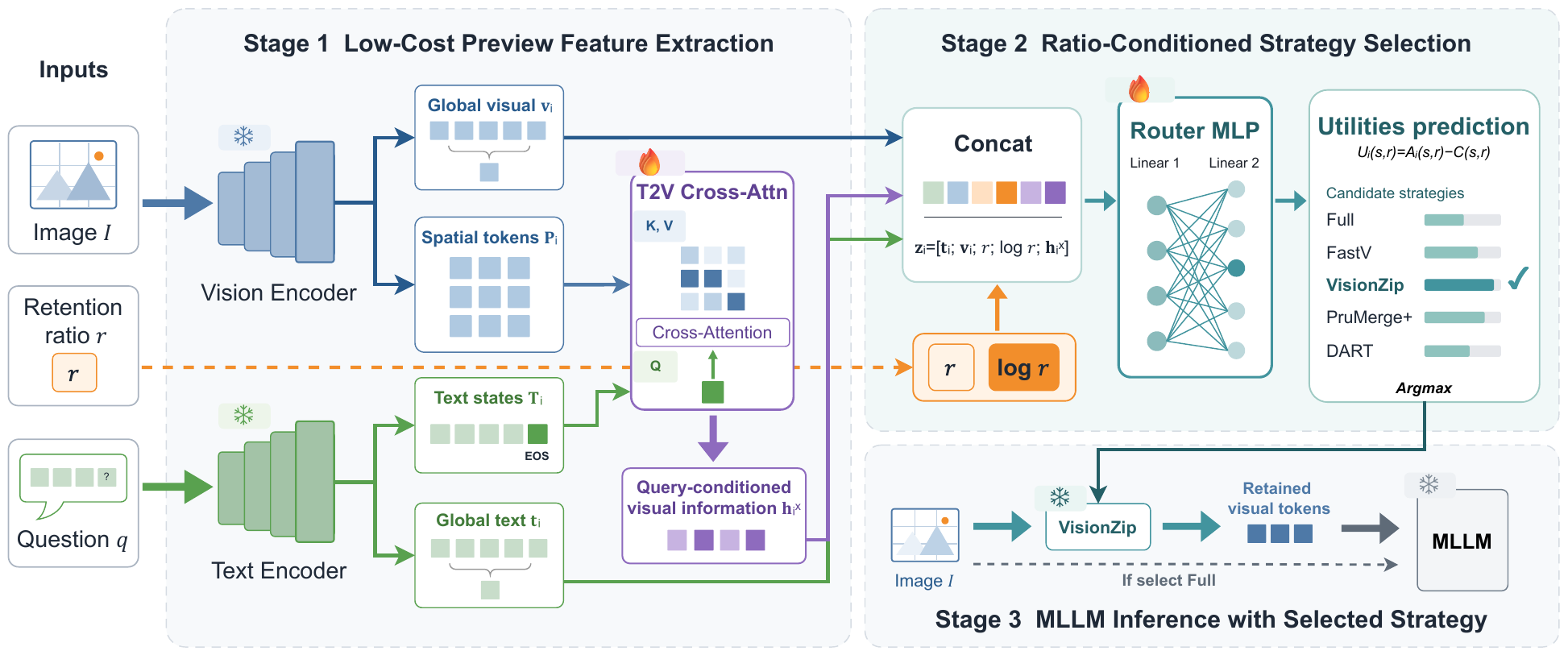}
\caption{\textbf{Overview of VIP-Router.} Frozen visual and textual preview features, together with query-conditioned cross-modal features and the retained-token ratio, are used to predict the utility of each candidate pruning strategy and Full.} 
\label{fig:framework} 
\end{figure} 

\subsection{Problem Formulation}
\label{sec:problem formulation}

We follow the Group~A setting of VTC-Bench~\citep{VTC-Bench_ACL2026}, where full-token inference is correct and the equivalent-ratio image-downsampling baseline fails.
Given a sample \(x_i=(I_i,q_i)\), we select a strategy from
\(\mathcal{S}=\{\mathrm{Full},\mathrm{FastV},\mathrm{VisionZip},
\mathrm{PruMerge+},\mathrm{DART}\}\).
The candidate \(\mathrm{Full}\) retains all visual tokens, while the remaining candidates apply existing pruning rules.

Let \(r\in(0,1]\) denote the retained-token ratio used by the compressed candidates; the corresponding reduction ratio is \(1-r\).
Here, \(r\) specifies the target pruning operating point rather than a hard per-sample token constraint, since VIP-Router may select \(\mathrm{Full}\) when pruning is unfavorable.
For strategy \(s\), let \(a_i^s\in\{0,1\}\) indicate whether the MLLM answers sample \(i\) correctly.

We define
\begin{equation}
U_i(s,r) = a_i^s-c(s,r), \qquad
c(s,r) =
\begin{cases}
1, & s=\mathrm{Full},\\
r, & s\neq\mathrm{Full}.
\end{cases}
\label{eq:utility}
\end{equation}
Under Group~A, this objective ranks a correctness-preserving compressed strategy above \(\mathrm{Full}\), and \(\mathrm{Full}\) above an incorrect compressed strategy, thereby providing a cost-aware no-pruning option.

The oracle sample-level decision is
\begin{equation}
s_i^\star = \arg\max_{s\in\mathcal{S}} U_i(s,r).
\label{eq:oracle}
\end{equation}
At test time, these utilities are unknown.
We therefore learn a router \(g_\theta\) that jointly predicts the utility of all candidates:
\begin{equation}
\hat{\mathbf{u}}_i = g_\theta(I_i,q_i,r),
\qquad
\hat{s}_i = \arg\max_{s\in\mathcal{S}} \hat{u}_{i,s}.
\label{eq:routing}
\end{equation}

\subsection{Cross-Modal Preview Representation} 
\label{sec:preview_representation} 

We instantiate the preview representation using frozen CLIP visual and text encoders.
For image \(I_i\), we extract penultimate-layer visual representations, yielding spatial patch features \(\mathbf{P}_i\), whose mean gives the global visual feature \(\mathbf{v}_i\).
For question \(q_i\), the text encoder produces token representations \(\mathbf{T}_i\) and a global text representation \(\mathbf{t}_i\).

While \(\mathbf{v}_i\) and \(\mathbf{t}_i\) provide low-cost global summaries, direct concatenation captures only coarse cross-modal interaction.
We therefore introduce a lightweight text-to-vision cross-attention module.
After projecting both modalities to a common dimension, the text tokens query the spatial visual features:
\begin{equation}
\mathbf{H}_i =
\operatorname{MHA}
\left(
Q=\mathbf{T}_i\mathbf{W}_T,\,
K=\mathbf{P}_i\mathbf{W}_P,\,
V=\mathbf{P}_i\mathbf{W}_P
\right).
\label{eq:cross_attention}
\end{equation}
We take the cross-attended EOS representation
\(\mathbf{h}_i^{\times}=\mathbf{H}_{i,\mathrm{EOS}}\)
as a compact query-conditioned visual summary.
The resulting representation therefore captures global question semantics,
global visual content, and query-conditioned visual information.

\subsection{Ratio-Conditioned Utility Routing} 
\label{sec:utility_routing} 

Training a separate router for each retained-token ratio would cause the number of routers to scale linearly with the number of supported ratios.
We instead share a single router across all ratios and explicitly condition it on \(r\).
The final router representation is
\begin{equation}
\mathbf{z}_i =
\left[
\mathbf{t}_i;
\mathbf{v}_i;
r;
\log r;
\mathbf{h}_i^{\times}
\right].
\label{eq:router_representation}
\end{equation}

A two-layer MLP maps \(\mathbf{z}_i\) to a utility estimate for each candidate strategy.
Predicting utility rather than correctness directly aligns the routing decision with the cost-aware objective in ~\eqref{eq:utility}, providing a common score for ranking candidates across retained-token ratios.

Training uses offline per-strategy correctness labels.
For each available sample--ratio pair, these labels are converted into the utility target
\[
\mathbf{U}_i =
[U_i(s_1,r),\ldots,U_i(s_K,r)]^\top.
\]
For stable optimization, we standardize each candidate's utility targets using training-set statistics and regress the resulting targets \(\widetilde{\mathbf U}_i\); predictions are transformed back to raw utility before strategy selection.
The complete standardization procedure is provided in Appendix~\ref{app:router_training_details}.
We optimize the router using an element-wise Huber loss:
\begin{equation}
\mathcal{L}(\theta)
=
\frac{1}{NK}
\sum_{i=1}^{N}
\sum_{k=1}^{K}
\ell_{\mathrm{Huber}}
\left(
\hat{\tilde{u}}_{i,k},
\widetilde{U}_{i,k}
\right).
\label{eq:training_loss}
\end{equation}

Only the feature projections, cross-attention module, and utility predictor are updated; the preview encoders and MLLM remain frozen.
At inference time, VIP-Router predicts candidate utilities at the specified retained-token ratio, selects the highest-utility candidate, and applies the selected strategy within the unchanged MLLM in a single forward pass.

\section{Experiments}
\label{sec:experiments}

\subsection{Experimental Setup}
\label{sec:experimental_setup}

\paragraph{Backbone and candidate pruning strategies.}
Following VTC-Bench~\citep{VTC-Bench_ACL2026}, we use Qwen2-VL-7B~\citep{Qwen2VL_2024} as the backbone MLLM in our main experiments.
We adopt the four pruning implementations released by VTC-Bench: FastV~\citep{FastV_ECCV2024}, VisionZip~\citep{VisionZip_CVPR2025}, PruMerge+~\citep{PruMerge_ICCV2025}, and DART~\citep{DART_EMNLP2025}.
All implementations and pruning hyperparameters are used without modification.
Full-token inference is included as an additional candidate for VIP-Router, as defined in Section~\ref{sec:problem formulation}.

\paragraph{Data construction and splits.}
VTC-Bench~\citep{VTC-Bench_ACL2026} constructs its pruning-sensitive Group~A evaluation set from eight established benchmarks: GQA~\citep{GQA_CVPR2019}, MMBench and MMBench-CN~\citep{MMBench_ECCV2024}, MME~\citep{MME_2024}, POPE~\citep{POPE_EMNLP2023}, MMStar~\citep{MMStar_NIPS2024}, OCRBench~\citep{OCRBench_SCIS2024}, and ChartQA~\citep{ChartQA_ACL2022}.
For each reduction ratio, Group~A contains samples that are answered correctly with full-token inference but incorrectly after equivalent-ratio image downsampling.
We construct the routing dataset from the correctness labels released by VTC-Bench and partition unique image--question records into training, validation, and test sets with a \(70/15/15\) split.
All ratio-specific examples derived from the same image--question record remain in the same split.
For each available sample--ratio example in the training set, the correctness vector over candidate strategies is converted into the utility target defined in~\eqref{eq:utility}.
The same split is used throughout all experiments, and the main results are reported on the held-out test set.

\paragraph{Baselines and evaluation metrics.}
We compare VIP-Router with each of the four pruning strategies applied uniformly to all test samples.
We additionally report \emph{Best Fixed}, which selects a single pruning strategy on the training split by maximizing average utility jointly over all training samples and reduction ratios, and then applies the same strategy to all test samples and ratios.
The \emph{Per-Sample Oracle} selects the highest-utility strategy using ground-truth test labels and serves only as an unattainable upper bound.
We report Accuracy, normalized Token Cost, Utility, and Oracle Headroom Recovery (OHR).
Accuracy measures the fraction of correctly answered Group~A samples.
Token Cost is the realized average fraction of retained visual tokens after strategy selection; because VIP-Router may select Full, its realized Token Cost can exceed the retained-token ratio of the compressed candidates.
Utility averages the sample-level objective in ~\eqref{eq:utility}.
OHR measures the fraction of the available sample-level utility headroom recovered by VIP-Router:
\begin{equation}
\mathrm{OHR}
=
\frac{
U_{\mathrm{VIP\text{-}Router}} - U_{\mathrm{best\text{-}fixed}}
}{
U_{\mathrm{oracle}} - U_{\mathrm{best\text{-}fixed}}
}
\times 100\%.
\label{eq:ohr}
\end{equation}
Higher OHR indicates that VIP-Router recovers a larger fraction of the utility gap between Best Fixed and the Per-Sample Oracle.

\paragraph{Router training.}
Unless otherwise stated, we use the CLIP-based router configuration described in Section~\ref{sec:preview_representation}.
Only the feature projections, cross-attention layer, and utility prediction head are optimized; the backbone MLLM and preview encoders remain frozen.
We train VIP-Router using AdamW with a learning rate of \(10^{-3}\), a batch size of 256, and at most 50 epochs, and select the checkpoint with the highest validation macro utility.
Additional optimization, implementation, hardware, and runtime details are provided in Appendix~\ref{app:imple and repro details}.

\begin{table*}[t]
\centering
\caption{
Cost-aware comparison of Best Fixed, VIP-Router, and the Per-Sample Oracle across reduction ratios on VTC-Bench Group~A test set.
VIP-Router results are averaged over five random seeds, whereas Best Fixed and the Per-Sample Oracle are deterministic.
Best Fixed uses a single pruning strategy selected jointly over all source-training samples and ratios.
For Best Fixed, Token Cost equals the exact retained-token ratio \(r\).
VIP-Router may select Full, and therefore its realized Token Cost can exceed \(r\).
Accuracy and utility are multiplied by 100 for reporting; token cost remains on the [0, 1] scale.
The best pruning result at each ratio is bolded.
}
\label{tab:summary_best_fixed_viprouter_oracle}
\footnotesize
\setlength{\tabcolsep}{4.8pt}
\renewcommand{\arraystretch}{1.12}

\begin{tabular*}{\textwidth}{
@{\extracolsep{\fill}}
c
!{\color{black!35}\vrule width 0.4pt}
c c c
!{\color{black!35}\vrule width 0.4pt}
c c c
!{\color{black!35}\vrule width 0.4pt}
c c c
!{\color{black!35}\vrule width 0.4pt}
c
@{\hspace{8pt}}
}
\toprule
\multirow{2}{*}{\textbf{Reduction Ratio}}
& \multicolumn{3}{c}{\textcolor{black!45}{\textbf{Oracle}}}
& \multicolumn{3}{c}{
    \begin{tabular}[c]{@{}c@{}}
        Best Fixed\\[-2pt]
        {\scriptsize\textit{VisionZip}}
    \end{tabular}
}
& \multicolumn{3}{c!{\color{black!35}\vrule width 0.4pt}}{\textbf{VIP-Router}}
& \multirow{2}{*}{\textbf{OHR}} \\
\cmidrule(lr){2-4}
\cmidrule(lr){5-7}
\cmidrule(lr){8-10}
& \textbf{Acc.}
& \textbf{Cost}
& \textbf{Utility}
& \textbf{Acc.}
& \textbf{Cost}
& \textbf{Utility}
& \textbf{Acc.}
& \textbf{Cost}
& \textbf{Utility}
& \\
\midrule

75.00\%
& \textcolor{black!45}{100.00}
& \textcolor{black!45}{0.351}
& \textcolor{black!45}{64.91}
& 57.54 & 0.250 & 32.54
& 66.91 & 0.287 & \textbf{38.23}
& 17.58\% \\

88.89\%
& \textcolor{black!45}{100.00}
& \textcolor{black!45}{0.353}
& \textcolor{black!45}{64.70}
& 47.94 & 0.111 & 36.83
& 57.65 & 0.144 & \textbf{43.24}
& 23.00\% \\

93.75\%
& \textcolor{black!45}{100.00}
& \textcolor{black!45}{0.390}
& \textcolor{black!45}{61.02}
& 37.42 & 0.063 & 31.17
& 49.62 & 0.095 & \textbf{40.08}
& 29.85\% \\

96.00\%
& \textcolor{black!45}{100.00}
& \textcolor{black!45}{0.460}
& \textcolor{black!45}{53.96}
& 33.05 & 0.040 & 29.05
& 44.29 & 0.073 & \textbf{36.95}
& 31.71\% \\

99.00\%
& \textcolor{black!45}{100.00}
& \textcolor{black!45}{0.586}
& \textcolor{black!45}{41.40}
& 25.82 & 0.010 & 24.82
& 37.48 & 0.076 & \textbf{29.89}
& 30.58\% \\

\bottomrule
\end{tabular*}
\end{table*}

\begin{table*}[t]
\centering
\caption{
Average accuracy (\%) on VTC-Bench Group~A test set across visual-token reduction
ratios.
Full has \(100\%\) accuracy by the construction of Group~A and is shown only
as a reference.
VIP-Router results are averaged over five random seeds, whereas fixed-strategy results are deterministic.
}
\label{tab:avg_acc_across_ratios}
\footnotesize
\setlength{\tabcolsep}{7.5pt}
\renewcommand{\arraystretch}{1.12}

\begin{tabular*}{\textwidth}{
@{\extracolsep{\fill}}
l
!{\color{black!35}\vrule width 0.4pt}
c c c c c
@{\hspace{8pt}}
}
\toprule
\textbf{Method}
& \textbf{75.00\%}
& \textbf{88.89\%}
& \textbf{93.75\%}
& \textbf{96.00\%}
& \textbf{99.00\%} \\
\midrule

\textcolor{black!45}{Full}
& \textcolor{black!45}{100.00}
& \textcolor{black!45}{100.00}
& \textcolor{black!45}{100.00}
& \textcolor{black!45}{100.00}
& \textcolor{black!45}{100.00} \\

\midrule
FastV
& 54.99 & 35.33 & 29.22 & 23.92 & 18.89 \\

VisionZip
& 57.54 & 47.94 & 37.42 & 33.05 & 25.82 \\

PruMerge\textsuperscript{+}
& 57.31 & 43.20 & 39.73 & 32.12 & 21.98 \\

DART
& 55.68 & 42.07 & 36.26 & 34.91 & 24.18 \\

\textbf{VIP-Router}
& \textbf{66.91}
& \textbf{57.65}
& \textbf{49.62}
& \textbf{44.29}
& \textbf{37.48} \\

\bottomrule
\end{tabular*}
\end{table*}

\subsection{Main Results}
\label{sec:main_results}

\paragraph{Cost-aware performance.}
Table~\ref{tab:summary_best_fixed_viprouter_oracle} shows that VIP-Router consistently improves over Best Fixed even after accounting for its realized visual-token cost.
Averaged across the five reduction ratios, utility increases from 30.88 to 37.68, corresponding to a 22.0\% relative improvement.
The gain holds at every reduction ratio, ranging from 5.07 to 8.91 utility points over Best Fixed.
As a stronger ratio-wise reference, the strongest fixed method itself varies with the pruning level:
PruMerge\textsuperscript{+} achieves the highest fixed-strategy utility at \(93.75\%\) reduction and DART at \(96.00\%\), while VisionZip leads at the remaining ratios.
VIP-Router nevertheless exceeds this fixed-method envelope at every ratio.
Relative to the Per-Sample Oracle, VIP-Router recovers 17.6--31.7\% of the available utility headroom across the five ratios, with a macro-average OHR of 26.5\%.
The improvement is also not primarily attributable to full-token inference.
Although VIP-Router may select Full when pruning is predicted to be unfavorable, its realized Token Cost is explicitly charged in utility, and Full accounts for only 3.48--6.66\% of routing decisions across the five ratios (Figure~\ref{fig:strategy_selection_distribution} in Appendix~\ref{appendix}).
Moreover, a worst-case counterfactual bounds the contribution of the full-token option to at most 0.41 utility points on average, or 6.1\% of the observed utility gain over Best Fixed (Remark~\ref{rem:full_fallback_bound}).
Thus, the improvement arises predominantly from sample-adaptive selection among pruning strategies rather than from additional use of full-token inference.

\paragraph{Accuracy across reduction ratios and benchmarks.}
VIP-Router also achieves the highest average accuracy at every reduction ratio, as shown in Table~\ref{tab:avg_acc_across_ratios}.
Compared with Best Fixed, it improves accuracy by at least 9.37 points at each ratio and increases the average accuracy across the five ratios from 40.35\% to 51.19\%, corresponding to a 26.9\% relative improvement.
The relative gain becomes particularly pronounced under aggressive pruning, reaching 45.2\% at \(99\%\) token reduction.
The improvement is also broadly distributed across tasks: VIP-Router achieves the best pruning result in 30 of the 40 benchmark--ratio combinations (75.0\%); the full per-benchmark breakdown is provided in Appendix~\ref{app:detailed_results}.

\subsection{Efficiency Analysis}
\label{sec:efficiency_analysis}

\begin{figure*}[t]
    \centering
    \begin{subfigure}[t]{0.485\textwidth}
        \centering
        \includegraphics[width=\linewidth]{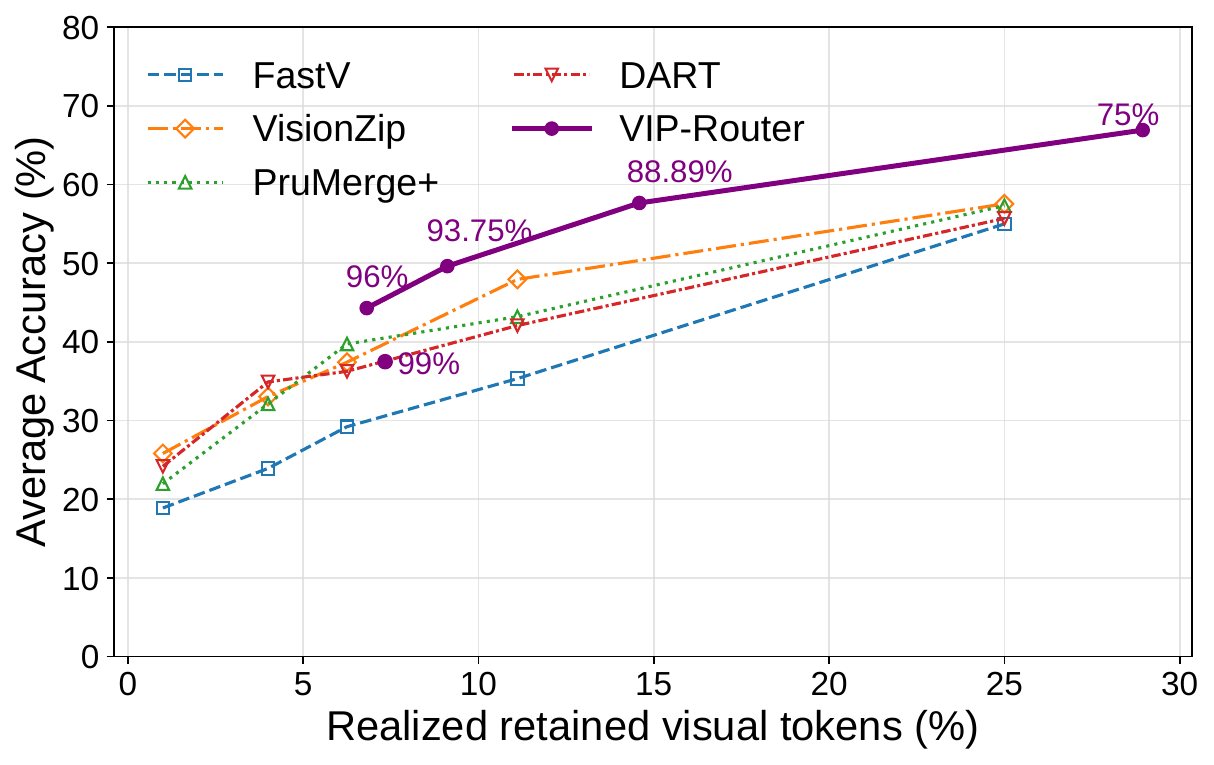}
        \caption{Accuracy versus realized token cost.}
        \label{fig:Token Budget-Acc.}
    \end{subfigure}
    \hfill
    \begin{subfigure}[t]{0.485\textwidth}
    \centering
        \includegraphics[width=\linewidth]{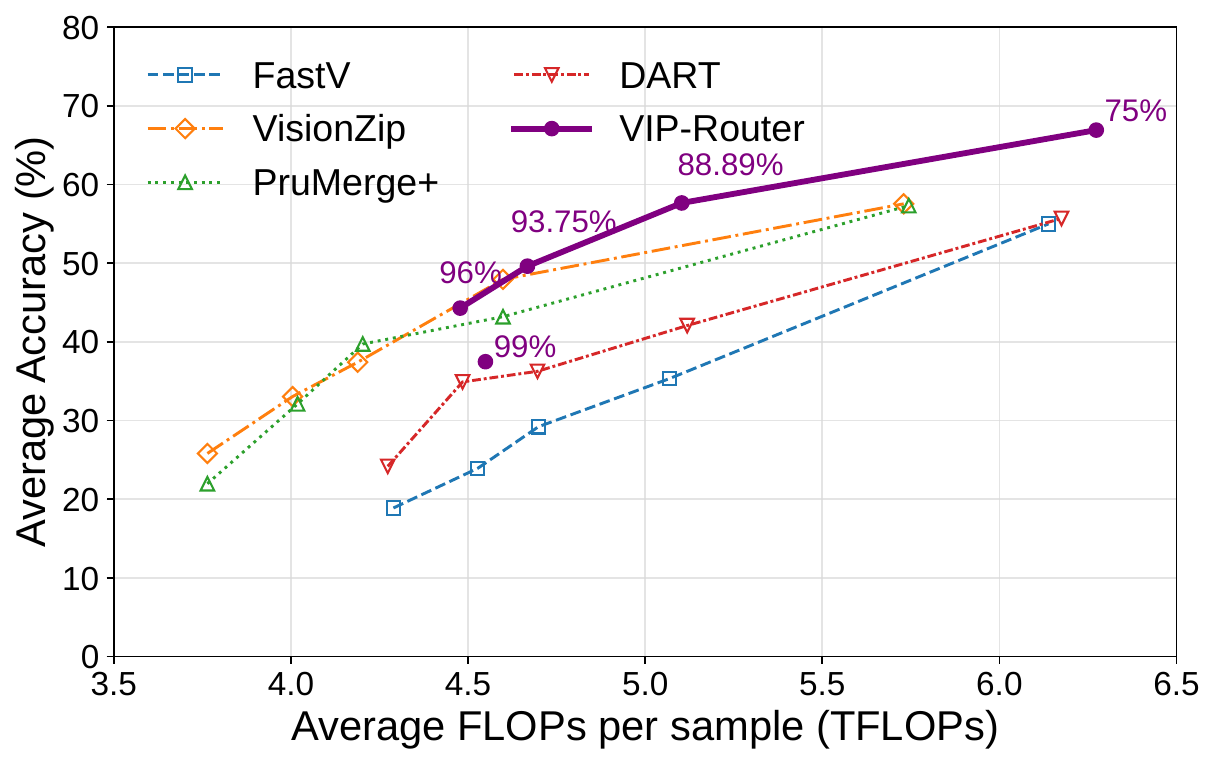}
        \caption{Accuracy versus total inference FLOPs.}
        \label{fig:FLOPs-Acc.}
    \end{subfigure}    
    \caption{
        Accuracy--efficiency comparison between fixed pruning strategies and VIP-Router. Point labels indicate prescribed reduction ratios.
        The total inference FLOPs include both routing overhead and the computation of the selected pruning strategy.}
    \label{fig:Accuracy_efficiency}

\end{figure*}

\begin{table*}[t]
    \centering    
    \caption{
Parameter overhead and performance of the default VIP-Router and its native-feature variants.
Extra frozen parameters include modules loaded exclusively for routing beyond Qwen2-VL; reused native MLLM components are excluded.
Percentages are relative to Qwen2-VL parameters.
All variants are single-run results.
}
    \label{tab:router_param_efficiency}
    \setlength{\tabcolsep}{7.5pt}
    \renewcommand{\arraystretch}{1.10}
    \resizebox{\textwidth}{!}{%
    \begin{tabular}{llccccc}
        \toprule
        Vision Features 
        & Text Features
        & Cross Attn.
        & \# Extra Frozen Params
        & \# Trainable Params
        & Acc.
        & Utility \\
        \midrule

        CLIP ViT-B/32
        & CLIP Text
        & Yes
        & 151.28\,M {\scriptsize (1.82\%)}
        & 1.38\,M {\scriptsize (0.017\%)}
        & 51.51
        & \textbf{38.09} \\

        Native ViT
        & CLIP Text
        & Yes
        & 63.43\,M {\scriptsize (0.77\%)}
        & 1.78\,M {\scriptsize (0.021\%)}
        & 51.57
        & 37.98 \\

        Native ViT
        & Native LM
        & Yes
        & 0.00\,M {\scriptsize (0.00\%)}
        & 4.13\,M {\scriptsize (0.050\%)}
        & 51.86
        & 37.37 \\

        Native ViT
        & CLIP Text
        & No
        & 63.43\,M {\scriptsize (0.77\%)}
        & \textbf{0.92\,M} {\scriptsize (0.011\%)}
        & 50.47
        & 37.19 \\

        Native ViT
        & Native LM
        & No
        & 0.00\,M {\scriptsize (0.00\%)}
        & 2.49\,M {\scriptsize (0.030\%)}
        & 52.21
        & 37.62 \\

        \bottomrule
    \end{tabular}%
    }
\end{table*}

\paragraph{Accuracy–efficiency trade-off.}
Figure~\ref{fig:Accuracy_efficiency} compares VIP-Router with fixed pruning strategies from two complementary perspectives: prescribed pruning level and realized computation.
As shown in Figure~\ref{fig:Token Budget-Acc.}, VIP-Router achieves higher accuracy than the fixed strategies across most of the evaluated reduction range.
Figure~\ref{fig:FLOPs-Acc.} further accounts for the FLOPs of both routing and the selected pruning method; VIP-Router remains above the fixed-strategy curves over most of the evaluated compute range beyond approximately \(4.5\) TFLOPs.
Together, these results show that the accuracy gains of VIP-Router persist after accounting for its additional routing computation.
The most aggressive \(99\%\) reduction setting behaves differently.
At this pruning level, VIP-Router more frequently selects the full-token option for samples predicted to be unfavorable for pruning, increasing its realized Token Cost and FLOPs.
Consequently, the corresponding point becomes less competitive under matched-compute comparison, even though VIP-Router still substantially improves accuracy over fixed strategies evaluated at the same prescribed reduction ratio.

\paragraph{Router parameter overhead.}
The default VIP-Router achieves the highest utility while introducing only 1.38M trainable parameters (0.017\% of the 8.291B Qwen2-VL \citep{Qwen2VL_2024}).
Although this configuration uses frozen CLIP encoders \citep{CLIP_ICML2021} for preview features, these external parameters are not intrinsic to the routing formulation: reusing native MLLM representations progressively reduces the additional frozen footprint and can eliminate it entirely.
In particular, the fully native variant without cross-attention requires no external frozen parameters and only 2.49M trainable parameters, while its utility is within 0.5 points of the default configuration.
Thus, VIP-Router can trade external preview capacity for native feature reuse while keeping its trainable parameter overhead at or below 0.05\% of the backbone.

\subsection{Cross-Backbone Evaluation and Benchmark Transfer}
\label{sec:generalization}

\begin{table*}[t]
\centering
\caption{
Evaluation of VIP-Router across MLLM backbones and zero-shot transfer to unseen benchmarks.
In panel (a), the pruning strategy selected by Best Fixed is shown after each MLLM name.
The constant Best Fixed Cost of \(0.095\) is the macro-average retained-token ratio over the five evaluated pruning levels (\(\bar r=0.0947\)).
Parenthesized percentages report OHR, which is also macro-averaged across the five reduction ratios.
}
\label{tab:generalization_results}

\begin{minipage}[t]{0.475\textwidth}
\centering
\scriptsize
\setlength{\tabcolsep}{2.6pt}
\renewcommand{\arraystretch}{1.05}

\begin{tabularx}{\linewidth}{
@{}
X
>{\centering\arraybackslash}p{0.12\linewidth}
>{\centering\arraybackslash}p{0.14\linewidth}
>{\centering\arraybackslash}p{0.30\linewidth}
@{}
}
\toprule
\textbf{Method} & \textbf{Acc.} & \textbf{Cost} & \textbf{Utility (OHR)} \\
\midrule
\rowcolor{blue!7}
\multicolumn{4}{@{}l}{\textit{LLaVA-1.5-7B (PruMerge\textsuperscript{+})}} \\
Oracle & 100.00 & 0.458 & 54.23 \\
Best Fixed & 43.38 & 0.095 & 33.91 \\
\textbf{VIP-Router} & 51.82 & 0.128 & \bestval{38.97}\ohrinline{26.15} \\

\midrule
\rowcolor{blue!7}
\multicolumn{4}{@{}l}{\textit{LLaVA-OneVision-7B (DART)}} \\
Oracle & 100.00 & 0.431 & 56.95 \\
Best Fixed & 45.22 & 0.095 & 35.74 \\
\textbf{VIP-Router} & 48.24 & 0.117 & \bestval{36.56}\ohrinline{2.93} \\

\midrule
\rowcolor{blue!7}
\multicolumn{4}{@{}l}{\textit{InternVL3-8B (DART)}} \\
Oracle & 100.00 & 0.385 & 61.46 \\
Best Fixed & 50.93 & 0.095 & 41.46 \\
\textbf{VIP-Router} & 53.81 & 0.118 & \bestval{42.01}\ohrinline{4.17} \\

\midrule
\rowcolor{blue!7}
\multicolumn{4}{@{}l}{\textit{Qwen2.5-VL-7B (DART)}} \\
Oracle & 100.00 & 0.450 & 54.96 \\
Best Fixed & 40.01 & 0.095 & 30.54 \\
\textbf{VIP-Router} & 48.75 & 0.134 & \bestval{35.40}\ohrinline{20.52} \\
\bottomrule
\end{tabularx}

{\normalsize\textbf{(a) Evaluation across MLLM backbones.}}
\end{minipage}
\hfill
\begin{minipage}[t]{0.475\textwidth}
\centering
\scriptsize
\setlength{\tabcolsep}{2.6pt}
\renewcommand{\arraystretch}{1.05}

\begin{tabularx}{\linewidth}{
@{}
X
>{\centering\arraybackslash}p{0.12\linewidth}
>{\centering\arraybackslash}p{0.14\linewidth}
>{\centering\arraybackslash}p{0.30\linewidth}
@{}
}
\toprule
\textbf{Method} & \textbf{Acc.} & \textbf{Cost} & \textbf{Utility (OHR)} \\
\midrule
\rowcolor{blue!7}
\multicolumn{4}{@{}l}{\textit{ScienceQA-IMG}} \\
Oracle & 100.00 & 0.518 & 48.16 \\
Best Fixed & 29.59 & 0.095 & 20.11 \\
\textbf{VIP-Router} & 36.18 & 0.106 & \bestval{25.54}\ohrinline{19.47} \\

\midrule
\rowcolor{blue!7}
\multicolumn{4}{@{}l}{\textit{MMMU}} \\
Oracle & 100.00 & 0.533 & 46.71 \\
Best Fixed & 31.99 & 0.095 & 22.52 \\
\textbf{VIP-Router} & 35.22 & 0.106 & \bestval{24.59}\ohrinline{6.53} \\

\midrule
\rowcolor{blue!7}
\multicolumn{4}{@{}l}{\textit{RealWorldQA}} \\
Oracle & 100.00 & 0.373 & 62.67 \\
Best Fixed & 38.41 & 0.095 & 28.94 \\
\textbf{VIP-Router} & 43.64 & 0.096 & \bestval{34.08}\ohrinline{13.62} \\

\midrule
\rowcolor{blue!7}
\multicolumn{4}{@{}l}{\textit{SEED-Bench Image}} \\
Oracle & 100.00 & 0.445 & 55.48 \\
Best Fixed & 39.27 & 0.095 & 29.80 \\
\textbf{VIP-Router} & 40.21 & 0.103 & \bestval{29.87}\ohrinline{0.47} \\
\bottomrule
\end{tabularx}

{\normalsize\textbf{(b) Transfer to unseen benchmarks.}}
\end{minipage}
\end{table*}

\paragraph{Evaluation across MLLM backbones.} 
We consider four representative MLLMs from different model families: LLaVA-1.5-7B, LLaVA-OneVision-7B, InternVL3-8B, and Qwen2.5-VL-7B~\citep{LLaVA-1.5_CVPR2024,LLaVA-OneVision_TMLR2024, InternVL3_2025,Qwen2.5VL_2025}. For each backbone, we run full-token, equivalent-downsampling, and per-strategy inference and reconstruct VTC-Bench Group~A following the same protocol described in Section~\ref{app:data_split}. We then train a separate VIP-Router instance using the corresponding backbone-specific source-training labels and evaluate it on the held-out test split. Best Fixed is selected from the same source-training split and held fixed across all reduction ratios. As shown in Table~\ref{tab:generalization_results}(a), VIP-Router improves utility over Best Fixed on all four backbones, with gains ranging from 0.55 to 5.06 points. The improvement varies substantially across models, indicating that the routing formulation is broadly applicable while the magnitude of sample-level pruning heterogeneity is backbone dependent. 

\paragraph{Zero-shot transfer to unseen benchmarks.} 
We further evaluate ScienceQA-IMG, MMMU, RealWorldQA, and SEED-Bench Image~\citep{ScienceQA_NIPS2022,MMMU_CVPR2024,RealWorldQA_2024, SEED-Bench_CVPR2024}. 
For each target benchmark,  the test set is reconstructed following the same protocol used for VTC-Bench Group~A. The VIP-Router trained on the VTC-Bench Group~A training set, together with the corresponding source-selected Best Fixed strategy, is then transferred directly to the reconstructed target Group~A without target-specific router training or strategy selection. 
As shown in Table~\ref{tab:generalization_results}(b), VIP-Router achieves positive utility gains on all four unseen benchmarks, although the transfer strength varies considerably across tasks. 
The largest gains occur on ScienceQA-IMG and RealWorldQA, whereas the improvement on SEED-Bench Image is marginal. 
Overall, these results show that pruning-strategy preferences learned from the source benchmarks can transfer to unseen tasks, but the degree of transfer remains sensitive to domain shift.

\subsection{Ablation Studies}

\begin{table*}[t]
    \centering
    \caption{
    Staged ablation of the VIP-Router design.
    All ablation results are reported from single runs, and comparisons are made within each stage.
    The best utility within each stage is bolded.
    }
    \label{tab:staged_ablation_summary}
    \scriptsize
    \setlength{\tabcolsep}{5.2pt}
    \renewcommand{\arraystretch}{1.10}

    \begin{tabular*}{\textwidth}{
        @{\extracolsep{\fill}}
        c c c c c c c c c c
        @{\hspace{8pt}}
    }
        \toprule
        \# Routers
        & $r$ Input
        & $T_{\text{global}}$
        & $V_{\text{global}}$
        & Spatial Info
        & Cross Attn.
        & Pooling
        & Acc.
        & Cost
        & Utility \\
        \midrule

        \rowcolor{orange!8}
        \multicolumn{10}{l}{\textit{A. Global feature inputs}} \\

        5 & -- & \checkmark & -- & -- & -- & -- & 47.99 & 0.122 & 35.81 \\
        5 & -- & -- & \checkmark & -- & -- & -- & 51.26 & 0.132 & 38.02 \\
        5 & -- & \checkmark & \checkmark & -- & -- & -- & 52.01 & 0.139 & \textbf{38.07} \\

        \midrule
        \rowcolor{orange!8}
        \multicolumn{10}{l}{\textit{B. Cross-ratio parameter sharing}} \\

        5 & -- & \checkmark & \checkmark & -- & -- & -- & 52.01 & 0.139 & \textbf{38.07} \\
        1 & \checkmark & \checkmark & \checkmark & -- & -- & -- & 49.59 & 0.126 & 36.96 \\

        \midrule
        \rowcolor{orange!8}
        \multicolumn{10}{l}{\textit{C. Query-conditioned spatial interaction}} \\

        1 & \checkmark & \checkmark & \checkmark & -- & -- & Mean
        & 48.29 & 0.119 & 36.41 \\

        1 & \checkmark & -- & \checkmark & \checkmark & -- & Learnable
        & 50.82 & 0.136 & 37.23 \\

        1 & \checkmark & \checkmark & \checkmark & \checkmark & -- & Learnable
        & 49.81 & 0.124 & 37.45 \\

        1 & \checkmark & \checkmark & \checkmark & \checkmark & \checkmark & Masked mean
        & 50.39 & 0.125 & 37.94 \\

        1 & \checkmark & \checkmark & \checkmark & \checkmark & \checkmark & EOS
        & 51.51 & 0.134 & \textbf{38.09} \\

        \bottomrule
    \end{tabular*}
\end{table*}

\paragraph{Global feature inputs.}
As shown in Table~\ref{tab:staged_ablation_summary}, global visual features provide a substantially stronger routing signal than global text features alone.
Combining the two yields the best performance in this stage, but the gain over using $V_{\text{global}}$ alone is marginal, suggesting that coarse global text features contribute little additional information through direct concatenation.

\paragraph{Cross-ratio parameter sharing.}
Replacing five ratio-specific routers with a single ratio-conditioned router reduces utility by only 1.11 points, while consolidating all pruning levels into one shared model.
This trade-off motivates the shared-router design used in VIP-Router, with the next stage focusing on recovering the performance loss through richer cross-modal interaction.

\paragraph{Query-conditioned cross-modal interaction.}
Spatially preserving the visual representation substantially improves the shared router over global averaging, whereas adding the global text feature alone provides only a modest additional gain.
Explicit text-to-vision cross-attention further improves routing performance, with the EOS-based representation achieving the best result.
The richer cross-modal representation more than compensates for the loss from parameter sharing. 
It achieves parity with five separate ratio-specific routers (38.09 vs. 38.07) while reducing the model overhead from five Stage-A models to just one, supporting the use of query-conditioned spatial interaction in the shared router.

\section{Conclusion}

In this paper, we introduced VIP-Router, a lightweight framework that reframes vision token pruning from applying a fixed criterion uniformly across inputs to selecting the most suitable pruning strategy for each sample.
Across five reduction ratios on the pruning-sensitive VTC-Bench Group~A setting, VIP-Router consistently improves both accuracy and utility over Best Fixed, remains effective across different MLLM backbones, and transfers zero-shot to unseen benchmarks.
Beyond these empirical gains, our results show that pruning-strategy choice itself constitutes an important axis of adaptive inference.

\subsection*{AI Use Statement}

In this work, we used generative AI tools to assist with software implementation and code review, data processing and table preparation, refinement of research ideas and experimental methodology, mathematical proofs production, and editing the manuscript for clarity and readability. Generative AI was not used to generate datasets; other required-disclosure tasks were not applicable to this work.

All AI-assisted code was manually reviewed and tested by the authors. AI-assisted data processing, tables, methodological suggestions, and manuscript edits were likewise checked against the underlying implementations, experimental results, and source materials. The authors made all final research and writing decisions and take full responsibility for the content of this work, including all text, claims, code, results, and artifacts produced with the assistance of generative AI.

\subsection*{Reproducibility Statement}

We provide the information required to reproduce VIP-Router and all reported experiments in the main paper, appendix, and supplementary materials.
The routing objective and model architecture are specified in Section~\ref{sec:method}, with pseudocode provided in Appendix~\ref{app:algorithm}.
Dataset construction, record-level train/validation/test splitting, baseline selection, and evaluation metrics are described in Section~\ref{sec:experimental_setup} and Appendix~\ref{app:imple and repro details}.
Appendix~\ref{app:imple and repro details} further reports the complete training configuration, random seeds, software environment, hardware, and implementation details.
Additional encoder configurations and detailed experimental results are provided in Appendices~\ref{app:encoder_analysis} and~\ref{app:additional_results}.

\subsection*{Limitations}

Our work has two main limitations:
(i) All empirical results in this paper are currently restricted to the Group A setting. Consequently, the router’s scaling behavior and performance on unfiltered benchmarks that include pruning-insensitive samples have not yet been comprehensively evaluated.
(ii) We have also conducted preliminary experiments to explore a stricter setting: whether a routing policy learned on a source MLLM can generalize directly to unseen target architectures without any target-specific adaptation. In this exploratory cross-backbone setup, a router trained solely on a source model is deployed on target backbones with all learned components and decision parameters held fixed.
Initial observations indicate that direct zero-shot policy transfer generally fails to preserve the advantages of backbone-specific training, often falling behind static baseline policies across target architectures. 
These findings suggest that while the proposed routing formulation is broadly applicable across diverse MLLM families, the optimal sample-level pruning preferences remain largely backbone-dependent. 

\bibliographystyle{bibstyle}
\bibliography{references}

\clearpage
\appendix
\section{Appendix}
\label{appendix}

\subsection{Theoretical Motivation: Utility Prediction from a Compact Representation}

\label{sec:theory}

\paragraph{Setup.}
Treat a benchmark sample \(x=(I,q)\) together with a retained-token ratio \(r\) as a random draw from the data distribution over \((x,r)\); we drop the sample index \(i\) used in Section~\ref{sec:method}.
The retained-token ratios lie on the finite grid
\[
\mathcal{R}=\{0.25,\,0.1111,\,0.0625,\,0.04,\,0.01\}.
\]
We write the candidate set as
\(\mathcal{S}=\{\mathrm{Full}\}\cup\mathcal{S}_{\mathrm{p}}\),
where \(\mathcal{S}_{\mathrm{p}}\) contains the pruning strategies.

Recall the token cost \(C(s,r)\) and per-sample utility
\[
U(s,r)=A(s,r)-C(s,r),
\qquad A(s,r)\in\{0,1\}.
\]
More generally, this objective belongs to the family
\(A(s,r)-\lambda C(s,r)\).
We use \(\lambda=1\) throughout the paper because, under Group~A, it induces the desired ordering between successful pruning, full-token inference, and failed pruning, as detailed below.

\paragraph{Structure of the oracle.}
Under Group~A, full-token inference is correct by construction, so
\[
U(\mathrm{Full},r)=0.
\]
All pruning strategies have the same token cost \(r\): a correct pruning strategy has utility \(1-r\), whereas an incorrect one has utility \(-r\).
The oracle therefore has a simple structure:
(i) if at least one pruning strategy is correct, every correct pruning strategy is optimal with utility \(1-r\);
(ii) if all pruning strategies fail, \(\mathrm{Full}\) is uniquely optimal with utility \(0\).
Thus, the per-sample optimum may contain multiple equivalent pruning strategies, and the router only needs to select one zero-regret member of this set.

\paragraph{What the router must learn.}
Let \(\mathbf{z}(x,r)\) denote the preview representation in ~\eqref{eq:router_representation}.
For a pruning strategy \(s\in\mathcal{S}_{\mathrm{p}}\), the Bayes-optimal squared-loss utility predictor is
\begin{equation}
\bar u_s(\mathbf{z})
=
\mathbb{E}[U(s,r)\mid\mathbf{z}]
=
\bar p_s(\mathbf{z})-r,
\qquad
\bar p_s(\mathbf{z})
=
\Pr(A(s,r)=1\mid\mathbf{z}).
\end{equation}
For \(\mathrm{Full}\), Group~A gives
\(\bar u_{\mathrm{Full}}=0\)
identically.
Utility prediction therefore reduces, in principle, to estimating the conditional success probabilities of the pruning strategies and comparing them against a known cost threshold.

At the Bayes optimum,
\begin{equation}
s^\star_{\mathbf{z}}=\mathrm{Full}
\quad\Longleftrightarrow\quad
\max_{s\in\mathcal{S}_{\mathrm{p}}}\bar p_s(\mathbf{z}) < r.
\end{equation}
Thus, full-token inference is preferred only when every pruning strategy is predicted to preserve correctness with probability below the retained-token ratio.
Our implementation predicts all candidate utilities using a common output head for simplicity, but the Full target remains constant under Group~A.

\paragraph{Error decomposition.}
Let $\mathbf{U}=[U(s_1,r),\dots,U(s_K,r)]$ denote the random utility vector and $\bar{\mathbf{u}}(\mathbf{z})=\mathbb{E}[\mathbf{U}\mid\mathbf{z}]$. For any square-integrable predictor $g$, the orthogonality of the $L^2$ projection onto $\sigma(\mathbf{z})$ yields
\begin{equation}
\mathbb{E}\big\|g(\mathbf{z})-\mathbf{U}\big\|_2^{2}
=
\underbrace{\mathbb{E}\big\|\bar{\mathbf{u}}(\mathbf{z})-\mathbf{U}\big\|_2^{2}}_{\delta_{\mathrm{rep}}^{2}\ \text{(representation gap)}}
\;+\;
\underbrace{\mathbb{E}\big\|g(\mathbf{z})-\bar{\mathbf{u}}(\mathbf{z})\big\|_2^{2}}_{\delta_{\mathrm{est}}^{2}\ \text{(estimation error)}}.
\end{equation}
The representation gap is the utility uncertainty that survives compressing $(x,r)$ into $\mathbf{z}$; it cannot be reduced by enlarging the predictor, only by enriching $\mathbf{z}$. Our three-branch design (global text, global vision, and query-conditioned spatial evidence) is designed to reduce this representation gap, and we validate it empirically by ablation.
\footnote{We train with Huber loss ($\delta{=}1$, ~\eqref{eq:training_loss}) for robustness. Each target $U(s,r)$ takes values in an interval of length one ($[-r,\,1-r]$ for $s\in\mathcal{S}_{\mathrm{p}}$; $[-1,0]$ for $\mathrm{Full}$). 
For the unstandardized utility targets, predictions inside the interval incur residuals in the quadratic regime, while predictions outside it are strictly dominated; the population minimizer therefore coincides with the conditional mean. We use this observation only to motivate utility regression, since our actual training objective applies Huber loss to standardized utility targets, for which exact equivalence to the conditional mean is not assumed.}

\paragraph{Approximation by a lightweight MLP.}
Since $\mathcal{R}$ is finite and the frozen encoders produce bounded features, $\mathbf{z}$ ranges over a compact set $\mathcal{Z}$. Assuming each $\bar u_s$ is continuous on $\mathcal{Z}$, universal approximation for non-polynomial activations~\citep{LeshnosTheorem_NN1993} implies that for any $\delta_{\mathrm{mlp}}>0$ there exists a finite-width two-layer GELU network $g_\theta$ such that
\begin{equation}
\sup_{\mathbf{z}\in\mathcal{Z}}\big\|g_\theta(\mathbf{z})-\bar{\mathbf{u}}(\mathbf{z})\big\|_\infty\le\delta_{\mathrm{mlp}}.
\end{equation}
This is an existence statement; optimization and generalization residuals of the trained router are absorbed into $\delta_{\mathrm{est}}$.

\paragraph{From utility prediction to decision quality.}
Let $\hat{\mathbf{u}}=g_\theta(\mathbf{z})$ and $\hat s\in\arg\max_{s\in\mathcal{S}}\hat u_s$.

\begin{lemma}[Pointwise regret]
\label{lem:regret}
If $\|\hat{\mathbf{u}}-\mathbf{U}\|_\infty\le\varepsilon$ on a sample $(x,r)$, then $U(s^\star)-U(\hat s)\le 2\varepsilon$.
\end{lemma}
\begin{proof}
$U(s^\star)\le \hat u_{s^\star}+\varepsilon\le \hat u_{\hat s}+\varepsilon\le U(\hat s)+2\varepsilon$.
\end{proof}

\begin{corollary}[Decision margin at aggressive pruning levels]
\label{cor:margin}
Let \(\mathcal{S}^\star(x,r)\) denote the set of optimal strategies and define the nonzero decision margin as the gap between the optimal utility and the best strictly suboptimal utility.
Under Group~A,
\[
\Delta(x,r)\in\{r,\,1-r\}.
\]
If at least one pruning strategy is correct, the gap between a correct pruning strategy and \(\mathrm{Full}\) is \(1-r\).
If all pruning strategies fail, the gap between \(\mathrm{Full}\) and a failed pruning strategy is \(r\).

Consequently, a uniform utility-prediction error
\(\varepsilon<\Delta(x,r)/2\)
is sufficient to preserve an optimal decision.
For
\(r\in\{0.04,0.0625,0.1111,0.25\}\),
\(\varepsilon<0.02\) is sufficient for all samples, whereas at
\(r=0.01\)
the corresponding guarantee requires
\(\varepsilon<0.005\).
The shrinking margin at the most aggressive pruning level occurs only at the boundary between \(\mathrm{Full}\) and failed pruning strategies; crossing this boundary incurs utility regret \(r\).
\end{corollary}

\begin{proposition}[Expected regret]
\label{prop:expected}
\begin{equation}
\mathbb{E}\big[\,U(s^\star)-U(\hat s)\,\big]
\;\le\;
2\,\mathbb{E}\big\|\hat{\mathbf{u}}-\mathbf{U}\big\|_\infty
\;\le\;
2\sqrt{\delta_{\mathrm{rep}}^{2}+\delta_{\mathrm{est}}^{2}},
\end{equation}
where the first inequality applies the proof of Lemma~\ref{lem:regret} pointwise, and the second uses $\|\cdot\|_\infty\le\|\cdot\|_2$ and Jensen's inequality.
\end{proposition}

\paragraph{Scope of the analysis.}
The analysis formalizes a limited claim: sample-adaptive pruning-strategy selection can be reduced to estimating conditional pruning success accurately enough to preserve its ordering relative to known token costs.
It does not establish that the preview representation is sufficient a priori.
This is an empirical question, which we evaluate through the feature ablations and the observed utility headroom recovered relative to the Per-Sample Oracle.

\begin{remark}[Upper bound on the contribution of full-token fallback]
\label{rem:full_fallback_bound}

Let \(f_r\) denote the fraction of samples for which VIP-Router selects
\(\mathrm{Full}\) at retained-token ratio \(r\).
Consider a counterfactual policy that disables \(\mathrm{Full}\) and instead
routes these samples to a compressed strategy.
In the worst case, all such replacements are incorrect.
Accuracy can therefore decrease by at most \(f_r\), while Token Cost
decreases by exactly \(f_r(1-r)\).
Consequently,
\begin{equation}
U_{\mathrm{forced\text{-}pruning}}
\ge
U_{\mathrm{VIP\text{-}Router}} - f_r r .
\label{eq:full_fallback_bound}
\end{equation}

Using the five-seed mean routing frequencies, the corresponding bounds are
\(1.23\), \(0.41\), \(0.22\), \(0.14\), and \(0.07\) utility points
at the five reduction ratios, respectively.
The average bound is only \(0.41\) utility points, corresponding to at most
\(6.1\%\) of VIP-Router's observed utility gain over Best Fixed.
Thus, even under this worst-case counterfactual, access to full-token
inference can explain only a small fraction of the overall improvement.
\end{remark}

\clearpage

\subsection{Algorithm}
\label{app:algorithm}

\begin{algorithm}[H]
\caption{Training procedure of VIP-Router.}
\label{alg:viprouter_training}
\begin{algorithmic}[1]
\Require Training set
\(\mathcal{D}_{\mathrm{train}}
= \{(I_i,q_i,r_i,\mathbf{a}_i)\}_{i=1}^{N}\),
candidate set \(\mathcal{S}=\{s_1,\ldots,s_K\}\),
frozen preview encoders,
and router \(g_\theta\)
\Ensure Trained router parameters \(\theta^\star\) and
training-set normalization statistics
\(\{(\mu_k,\sigma_k)\}_{k=1}^{K}\)

\State Construct utility targets for all training samples:
\[
U_i(s_k,r_i)
\gets
a_i^{s_k}-c(s_k,r_i),
\qquad
i=1,\ldots,N,\;\; k=1,\ldots,K
\]

\State Compute per-strategy normalization statistics on the training set:
\[
\mu_k
\gets
\frac{1}{N}\sum_{i=1}^{N} U_i(s_k,r_i),
\qquad
\sigma_k
\gets
\sqrt{
\frac{1}{N}\sum_{i=1}^{N}
\bigl(U_i(s_k,r_i)-\mu_k\bigr)^2
},
\quad k=1,\ldots,K
\]

\For{each training epoch}
    \For{each minibatch \(\mathcal{B}\subset\mathcal{D}_{\mathrm{train}}\)}
        \For{each \((I_i,q_i,r_i,\mathbf{a}_i)\in\mathcal{B}\)}
            \State Standardize the utility target for each strategy:
            \[
            \widetilde{U}_{i,k}
            \gets
            \frac{U_i(s_k,r_i)-\mu_k}{\sigma_k},
            \qquad k=1,\ldots,K
            \]

            \State Form the standardized target utility vector:
            \[
            \widetilde{\mathbf{U}}_i
            \gets
            [\widetilde{U}_{i,1},\ldots,\widetilde{U}_{i,K}]^\top
            \]

            \State Extract frozen visual and textual preview features from \((I_i,q_i)\)
            \State Compute query-conditioned cross-modal features

            \State Predict standardized strategy utilities:
            \[
            \hat{\widetilde{\mathbf{u}}}_i
            \gets
            g_\theta(I_i,q_i,r_i)
            \]
        \EndFor

        \State Compute the element-wise Huber loss:
        \[
        \mathcal{L}
        \gets
        \frac{1}{|\mathcal{B}|K}
        \sum_{i\in\mathcal{B}}
        \sum_{k=1}^{K}
        \ell_{\mathrm{Huber}}
        \bigl(
        \hat{\widetilde{u}}_{i,k},
        \widetilde{U}_{i,k}
        \bigr)
        \]

        \State Update \(\theta\) by minimizing \(\mathcal{L}\)
    \EndFor

    \State Evaluate utility on the validation set after recovering utilities to the original scale
\EndFor

\State \Return checkpoint \(\theta^\star\) with the highest validation utility
and \(\{(\mu_k,\sigma_k)\}_{k=1}^{K}\)
\end{algorithmic}
\end{algorithm}

\begin{algorithm}[H]
\caption{Inference procedure of VIP-Router.}
\label{alg:viprouter_inference}
\begin{algorithmic}[1]
\Require Image \(I\), question \(q\), retained-token ratio \(r\),
candidate set \(\mathcal{S}=\{s_1,\ldots,s_K\}\),
trained router \(g_{\theta^\star}\),
training-set normalization statistics
\(\{(\mu_k,\sigma_k)\}_{k=1}^{K}\),
and frozen MLLM \(f\)

\State Extract frozen visual and textual preview features from \((I,q)\)
\State Compute query-conditioned cross-modal features

\State Predict standardized strategy utilities:
\[
\hat{\widetilde{\mathbf{u}}}
\gets
g_{\theta^\star}(I,q,r)
\]

\State Recover predicted utilities on the original scale:
\[
\hat{u}_k
\gets
\sigma_k \hat{\widetilde{u}}_k + \mu_k,
\qquad k=1,\ldots,K
\]

\State Select the highest-utility candidate:
\[
\hat{s}
\gets
s_{\arg\max_{k\in\{1,\ldots,K\}}\hat{u}_k}
\]

\State Apply \(\hat{s}\) and run the frozen MLLM once:
\[
\hat{y}
\gets
f_{\hat{s}}(I,q;r)
\]

\State \Return \(\hat{y}\)
\end{algorithmic}
\end{algorithm}

\subsection{Implementation and Reproducibility Details}
\label{app:imple and repro details}

\subsubsection{Dataset Construction and Split Details}
\label{app:data_split}

We construct the routing dataset from the per-sample inference results officially released by VTC-Bench~\citep{VTC-Bench_ACL2026} for
Qwen2-VL-7B-Instruct.
The released results include full-token inference, equivalent-ratio image downsampling, and the four candidate pruning methods evaluated at multiple reduction ratios.
For each benchmark, binary correctness is derived from its task-specific evaluation fields.
For MMBench and MMBench-CN, we compare the extracted multiple-choice prediction with the ground-truth answer.

Following VTC-Bench, we reconstruct Group~A independently at each reduction ratio.
An image--question record is included at a given ratio when full-token
inference is correct while equivalent-ratio image downsampling is incorrect.
We then associate the correctness labels of FastV, VisionZip, PruMerge\textsuperscript{+}, and DART with each eligible record--ratio pair.
This procedure produces 33,091 routing examples from 12,919 unique
image--question records.
Because Group~A membership depends on the reduction ratio, the routing data form a sparse record--ratio grid: the same record may appear at only a subset of the evaluated ratios.

We evaluate five retained-token ratios,
\[
r\in\{0.25,\;0.1111,\;0.0625,\;0.04,\;0.01\},
\]
corresponding to visual-token reduction ratios of \(75.00\%\), \(88.89\%\), \(93.75\%\), \(96.00\%\), and \(99.00\%\), respectively.
The ratio inputs to VIP-Router are \(r\) and \(\ln(r+10^{-8})\).

Data splitting is performed at the image--question-record level rather than at the record--ratio level, ensuring that all ratio-specific examples of the same record remain in the same split.
Records are identified by the pair \texttt{benchmark:doc\_id} and assigned
deterministically to training, validation, and test sets using a fixed
hash-based \(70/15/15\) split.
This results in 8,988 training, 1,956 validation, and 1,975 test records,
with no record shared across splits.
The split is record-disjoint but not necessarily image-disjoint: when the same underlying image is paired with multiple questions and therefore multiple document IDs, the resulting image--question records are treated as distinct samples.

\begin{table}[t]
\centering
\caption{
Number of routing examples at each visual-token reduction ratio.
Group~A membership is ratio-dependent, so the number of eligible examples
varies across ratios.
}
\label{tab:appendix_dataset_stats}
\small
\setlength{\tabcolsep}{6pt}
\begin{tabular}{ccrrrr}
\toprule
Reduction & Retained \(r\) & Train & Val. & Test & Total \\
\midrule
75.00\%  & 0.2500 & 1,998 & 411   & 431   & 2,840 \\
88.89\%  & 0.1111 & 3,720 & 810   & 801   & 5,331 \\
93.75\%  & 0.0625 & 4,772 & 1,044 & 1,037 & 6,853 \\
96.00\%  & 0.0400 & 5,391 & 1,136 & 1,183 & 7,710 \\
99.00\%  & 0.0100 & 7,218 & 1,551 & 1,588 & 10,357 \\
\midrule
Total    & --     & 23,099 & 4,952 & 5,040 & 33,091 \\
\bottomrule
\end{tabular}
\end{table}

\subsubsection{Router Architecture and Training Configuration}
\label{app:router_training_details}

\paragraph{Preview encoder.}
The default VIP-Router uses the frozen \texttt{openai/clip-vit-base-patch32} checkpoint for both visual and textual preview encoding.
Images follow the default CLIP preprocessing and are resized and
center-cropped to \(224\times224\).
We extract the penultimate visual hidden representation, discard the CLS token, and retain the resulting 49 patch tokens \(\mathbf{P}\in\mathbb{R}^{49\times768}\).
Their mean forms the 768-dimensional global visual representation.

Questions are tokenized to a maximum length of 77 tokens.
We retain the 512-dimensional text-token representations and use the
CLIP-projected EOS representation as the 512-dimensional global text feature.
All CLIP parameters remain frozen during router training.

\paragraph{Cross-modal representation and utility prediction.}
Text and visual token representations are projected independently to
256 dimensions.
A single four-head text-to-vision cross-attention layer uses the projected text tokens as queries and the projected visual tokens as keys and values.
The attention output is followed by LayerNorm, and the EOS-position output is used as the 256-dimensional query-conditioned visual representation \(\mathbf{h}^{\times}\).
The attention layer uses dropout \(0.1\); no additional residual or
feed-forward block is introduced.

The router input is
\[
\mathbf{z}
=
[\mathbf{t}_{\mathrm{global}};
 \mathbf{v}_{\mathrm{global}};
 r;
 \ln(r+10^{-8});
 \mathbf{h}^{\times}]
\in\mathbb{R}^{1538}.
\]
A two-layer MLP,
\[
1538 \rightarrow 512 \rightarrow 5,
\]
with GELU activation and dropout \(0.1\), predicts the utilities of
\[
\{\mathrm{FastV},\mathrm{VisionZip},
  \mathrm{PruMerge}^{+},\mathrm{DART},\mathrm{Full}\}.
\]
Full is treated as an ordinary candidate whose utility is predicted by the router rather than analytically fixed.
The complete router contains 1,382,405 trainable parameters.

\paragraph{Standardization.}
All standardization statistics are estimated exclusively from the training split.
The frozen global preview features and the two ratio features are
standardized feature-wise using their training-set means and standard
deviations.

Utility targets are standardized separately for each candidate.
For candidate \(k\), let \(\mu_k\) and \(\sigma_k\) denote the mean and standard deviation of its training-set utility.
The regression target is
\begin{equation}
\widetilde{U}_{i,k}
=
\frac{U_{i,k}-\mu_k}{\sigma_k}.
\label{eq:utility_standardization}
\end{equation}
For constant dimensions, including the Full utility under Group~A, the
standardization scale is set to \(1\).
At inference time, predicted standardized utilities are transformed back to the original utility scale,
\begin{equation}
\hat U_{i,k}
=
\hat{\widetilde U}_{i,k}\sigma_k+\mu_k,
\end{equation}
and the candidate with the largest predicted utility is selected.
Validation and test samples are not used to estimate any standardization statistics.

\paragraph{Optimization.}
We train the projection layers, cross-attention module, LayerNorm, and utility predictor using AdamW with learning rate \(10^{-3}\), weight decay \(10^{-4}\), and batch size 256.
Training runs for at most 50 epochs with a linear warm-up over the first \(5\%\) of optimization steps followed by cosine learning-rate decay.
Gradients are clipped to a global norm of \(1.0\).

The objective is an element-wise Huber loss with \(\delta=1\), averaged over
samples and the five candidate utilities in the standardized target space.
Training is performed in FP32.
We use early stopping with patience five.
For checkpoint selection, validation utility is first averaged within each retained-token ratio and then macro-averaged across the five ratios.
The checkpoint with the highest validation macro utility is used for final test evaluation.

Frozen CLIP representations are precomputed once for each unique
image--question record and reused across its available retained-token ratios; all trainable VIP-Router components are evaluated online.

Unless otherwise stated, VIP-Router results are averaged over five independent training runs with random seeds \(0,1,2,3,4\).
All runs use the same deterministic data split, cached preview features, architecture, target standardization, and optimization configuration; only the random training trajectory differs.

\subsubsection{Hardware and Software Environment}
\label{app:hardware_software}

\begin{sloppypar}
All experiments are conducted on a single NVIDIA H200 GPU.
We use Python~3.10.12, PyTorch~2.3.0 with CUDA~12.4, and
Transformers~4.49.0.
The CLIP preview encoder is implemented with \texttt{transformers.\allowbreak CLIPModel} and \texttt{CLIPProcessor}.
Router training is performed in FP32.
\end{sloppypar}

\subsection{Encoder Analysis}
\label{app:encoder_analysis}

\begin{table*}[t]
\centering
\caption{
Performance of VIP-Router with alternative frozen preview encoders.
The default CLIP-B/32 result is averaged over five random seeds, whereas the
alternative encoder variants are single-run results and are therefore intended
as robustness evidence rather than a strict encoder ranking.
}
\label{tab:encoder_analysis}
\footnotesize
\setlength{\tabcolsep}{5.0pt}
\renewcommand{\arraystretch}{1.10}

\begin{tabular*}{\textwidth}{
@{\extracolsep{\fill}}
l l l c c c
@{}
}
\toprule
\textbf{Variant}
& \textbf{Vision Encoder}
& \textbf{Text Encoder}
& \textbf{Acc.}
& \textbf{Cost}
& \textbf{Utility} \\
\midrule

Default
& CLIP ViT-B/32
& CLIP Text
& 51.19
& 0.135
& 37.68 \\

CLIP-B/16
& CLIP ViT-B/16
& CLIP Text
& 50.31
& 0.130
& 37.28 \\

SigLIP2-B/32
& SigLIP2-B/32
& SigLIP2 Text
& 47.38
& 0.110
& 36.36 \\

TIPSv2-B/14
& TIPSv2-B/14
& TIPSv2 Text
& 49.87
& 0.125
& 37.34 \\

EUPE + CLIP Text
& EUPE-ViT-B
& CLIP Text
& 51.12
& 0.132
& 37.96 \\

EUPE + SigLIP2 Text
& EUPE-ViT-B
& SigLIP2 Text
& 50.47
& 0.131
& 37.34 \\

\bottomrule
\end{tabular*}
\end{table*}

\paragraph{Encoder alternative.}
The default VIP-Router uses CLIP ViT-B/32 and its paired text encoder to construct the preview representation.
To examine whether the routing framework remains effective with alternative preview encoders, we replace the default frozen encoders while keeping the routing objective, candidate strategy set, dataset split, and architecture unchanged.
All variants retain the same text-to-vision cross-attention module, EOS-style pooling, ratio conditioning, and two-layer utility predictor.

As shown in Table~\ref{tab:encoder_analysis}, VIP-Router remains effective across substantially different preview encoder families.
CLIP-B/16 \citep{CLIP_ICML2021}, TIPSv2-B/14 \citep{TIPSv2_CVPR2026}, and both EUPE \citep{EUPE_2026} variants achieve utility close to the default CLIP-B/32 configuration, while SigLIP2-B/32 \citep{SigLIP2_2025} exhibits a somewhat larger decrease.
Notably, the EUPE--CLIP variant achieves slightly higher utility than the default configuration, with comparable accuracy.
These results indicate that the routing formulation is not restricted to a single preview feature family and can operate with multiple frozen vision--language representations.

We do not interpret the small differences among encoder variants as a strict ranking.
The alternative configurations were evaluated as single runs, and their training RNG trajectories are not perfectly matched across variants.
Accordingly, this experiment is intended to assess the robustness of VIP-Router to encoder replacement rather than to identify an optimal preview encoder.

\subsection{Additional Results}
\label{app:additional_results}

\begin{table*}[!htbp]
  \centering
  \caption{
Per-benchmark accuracy (\%) of fixed pruning strategies and VIP-Router on VTC-Bench Group~A test set across visual-token reduction ratios.
Full is shown in gray as a reference and has \(100\%\) accuracy by the construction of Group~A; it is excluded from highlighting.
Fixed-strategy results are deterministic, whereas VIP-Router results are averaged over five random seeds.
Within each reduction-ratio block, the best pruning result in each column is bolded, and the second-best result in the Avg. column is underlined.
  }
  \label{tab:per_benchmark_accuracy}
  \footnotesize
  \setlength{\tabcolsep}{5.0pt}
  \renewcommand{\arraystretch}{1.16}
  \resizebox{\textwidth}{!}{%
  \begin{tabular}{
    @{}
    l
    !{\color{black!35}\vrule width 0.4pt}
    *{8}{c}
    !{\color{black!35}\vrule width 0.4pt}
    c
    @{\hspace{6pt}}
  }
    \toprule
    \textbf{Method}
    & \textbf{GQA}
    & \textbf{MMB}
    & \textbf{MMB}\textsuperscript{CN}
    & \textbf{MME}
    & \textbf{POPE}
    & \textbf{MMStar}
    & \textbf{OCR}
    & \textbf{ChartQA}
    & \textbf{Overall} \\
    \midrule

    \textcolor{black!45}{Full}
    & \textcolor{black!45}{100.00}
    & \textcolor{black!45}{100.00}
    & \textcolor{black!45}{100.00}
    & \textcolor{black!45}{100.00}
    & \textcolor{black!45}{100.00}
    & \textcolor{black!45}{100.00}
    & \textcolor{black!45}{100.00}
    & \textcolor{black!45}{100.00}
    & \textcolor{black!45}{100.00} \\
    \midrule

    \rowcolor{gray!12}
    \multicolumn{10}{c}{\textit{Reduction ratio = 75.00\%}} \\[1pt]
    FastV
    & 61.59 & 54.05 & 43.33 & 87.50 & 69.44 & 55.88 & 25.81 & 39.71 & 54.99 \\
    VisionZip
    & 60.26 & 54.05 & 33.33 & 62.50 & 80.56 & \textbf{58.82} & 19.35 & \textbf{55.88} & \underline{57.54} \\
    PruMerge\textsuperscript{+}
    & \textbf{64.24} & 62.16 & 40.00 & 62.50 & 77.78 & 52.94 & 22.58 & 42.65 & 57.31 \\
    DART
    & 56.95 & 51.35 & 46.67 & 87.50 & 73.61 & 50.00 & 51.61 & 41.18 & 55.68 \\
    \rowcolor{orange!8}
    \textbf{VIP-Router}
    & 64.11 & \textbf{63.24} & \textbf{66.67} & \textbf{90.00}
    & \textbf{90.83} & 50.59 & \textbf{56.77} & 47.65 & \textbf{66.91} \\
    \midrule

    \rowcolor{gray!12}
    \multicolumn{10}{c}{\textit{Reduction ratio = 88.89\%}} \\[1pt]
    FastV
    & 42.45 & 36.21 & 28.81 & 41.18 & 43.84 & 25.00 & 10.00 & 32.26 & 35.33 \\
    VisionZip
    & 46.70 & 27.59 & 33.90 & 41.18 & 74.66 & 28.12 & 20.00 & \textbf{51.61} & \underline{47.94} \\
    PruMerge\textsuperscript{+}
    & 47.64 & 25.86 & 22.03 & 47.06 & 71.23 & 31.25 & 10.00 & 41.01 & 43.20 \\
    DART
    & 49.06 & 32.76 & 37.29 & \textbf{58.82} & 60.96 & \textbf{46.88} & 33.33 & 26.73 & 42.07 \\
    \rowcolor{orange!8}
    \textbf{VIP-Router}
    & \textbf{53.58} & \textbf{46.90} & \textbf{49.15} & 57.65
    & \textbf{83.70} & 34.38 & \textbf{47.00} & 51.52 & \textbf{57.65} \\
    \midrule

    \rowcolor{gray!12}
    \multicolumn{10}{c}{\textit{Reduction ratio = 93.75\%}} \\[1pt]
    FastV
    & 38.08 & 36.49 & 25.00 & \textbf{52.94} & 38.71 & 31.82 & 8.11 & 15.22 & 29.22 \\
    VisionZip
    & 38.43 & 28.38 & 27.94 & 38.24 & 70.97 & 31.82 & 8.11 & 27.17 & 37.42 \\
    PruMerge\textsuperscript{+}
    & 46.26 & 27.03 & 22.06 & 38.24 & 73.12 & 31.82 & 9.46 & 27.90 & \underline{39.73} \\
    DART
    & 44.48 & 25.68 & 39.71 & 44.12 & 59.14 & \textbf{34.09} & 31.08 & 15.22 & 36.26 \\
    \rowcolor{orange!8}
    \textbf{VIP-Router}
    & \textbf{49.82} & \textbf{43.24} & \textbf{46.47} & 48.24
    & \textbf{85.27} & 31.36 & \textbf{43.51} & \textbf{29.35} & \textbf{49.62} \\
    \midrule

    \rowcolor{gray!12}
    \multicolumn{10}{c}{\textit{Reduction ratio = 96.00\%}} \\[1pt]
    FastV
    & 28.49 & 28.42 & 27.62 & 27.59 & 37.32 & 31.48 & 7.50 & 8.03 & 23.92 \\
    VisionZip
    & 37.09 & 33.68 & 28.57 & 37.93 & 66.99 & 33.33 & 6.25 & 10.95 & 33.05 \\
    PruMerge\textsuperscript{+}
    & 39.17 & 25.26 & 22.86 & 34.48 & 62.20 & 24.07 & 5.00 & 15.69 & 32.12 \\
    DART
    & 41.84 & 24.21 & 39.05 & \textbf{44.83} & 56.46 & 35.19 & 31.25 & 12.04 & \underline{34.91} \\
    \rowcolor{orange!8}
    \textbf{VIP-Router}
    & \textbf{43.50} & \textbf{40.63} & \textbf{46.48} & 40.69
    & \textbf{78.47} & \textbf{35.93} & \textbf{41.50} & \textbf{20.07} & \textbf{44.29} \\
    \midrule

    \rowcolor{gray!12}
    \multicolumn{10}{c}{\textit{Reduction ratio = 99.00\%}} \\[1pt]
    FastV
    & 18.57 & 14.71 & 20.73 & 19.35 & 45.49 & 15.87 & 2.73 & 3.28 & 18.89 \\
    VisionZip
    & 24.90 & 26.47 & 34.15 & 33.87 & 51.76 & 28.57 & 0.00 & 5.84 & \underline{25.82} \\
    PruMerge\textsuperscript{+}
    & 22.04 & 15.29 & 16.46 & 24.19 & 53.33 & 23.81 & 1.82 & 7.30 & 21.98 \\
    DART
    & 24.69 & 19.41 & 29.27 & \textbf{37.10} & 38.43 & 26.98 & 30.91 & 3.65 & 24.18 \\
    \rowcolor{orange!8}
    \textbf{VIP-Router}
    & \textbf{27.14} & \textbf{37.76} & \textbf{43.66} & 36.45
    & \textbf{61.57} & \textbf{32.06} & \textbf{41.27} & \textbf{25.47} & \textbf{37.48} \\
    \bottomrule
  \end{tabular}%
  }
\end{table*}

\subsubsection{Detailed Per-Benchmark Results}
\label{app:detailed_results}

Table~\ref{tab:per_benchmark_accuracy} reports the complete per-benchmark accuracy results underlying the aggregate comparison in Section~\ref{sec:main_results}.
VIP-Router achieves the best pruning result in 30 of the 40 benchmark--ratio combinations (75.0\%), showing that its improvement is broadly distributed across benchmarks rather than driven by a small subset of tasks.
Its gains also persist under the most aggressive reduction settings, where the accuracy of all fixed pruning strategies degrades substantially.
Full-token inference is included only as a reference and has \(100\%\) accuracy by the construction of VTC-Bench Group~A.
The strongest fixed strategy varies with the pruning level:
VisionZip achieves the highest overall accuracy among fixed strategies at \(75.00\%\), \(88.89\%\), and \(99.00\%\) reduction, whereas PruMerge\textsuperscript{+} and DART are strongest at \(93.75\%\) and \(96.00\%\), respectively.
VIP-Router nevertheless achieves higher overall accuracy than the strongest fixed strategy at every evaluated reduction ratio.

\begin{figure*}[t]
    \centering
    \includegraphics[width=\textwidth]{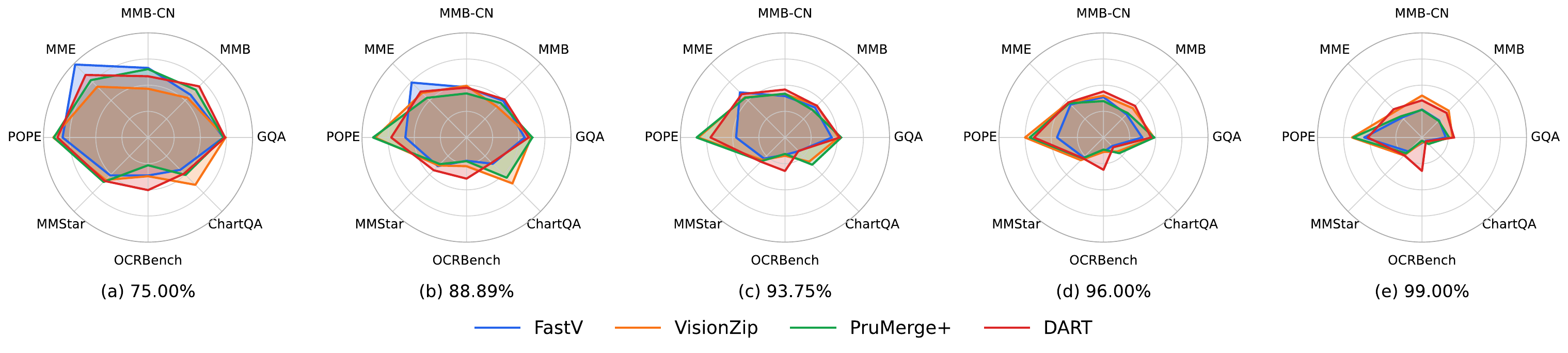}
    \caption{
   Per-benchmark accuracy of the four fixed pruning strategies on VTC-Bench Group A.
    }
    \label{fig:vtc_groupa_radar}
\end{figure*}

\subsubsection{Complementarity among Pruning Strategies}
\label{app:strategy_complementarity_analysis}

Figure~\ref{fig:vtc_groupa_radar} provides the per-benchmark comparison of the fixed pruning strategies. Their relative performance varies across benchmarks, and no single pruning criterion is uniformly strongest. Across reduction ratios, a substantial fraction of Group~A samples for which Best Fixed fails can still be correctly handled by at least one alternative pruning strategy. Thus, failures of the globally selected fixed strategy are frequently recoverable by another pruning criterion. This sample-level evidence complements the Best Fixed--Oracle gap reported in the main text and shows that benchmark-average rankings mask substantial strategy complementarity.

\begin{figure*}[t]
    \centering
    \includegraphics[width=0.7\textwidth]{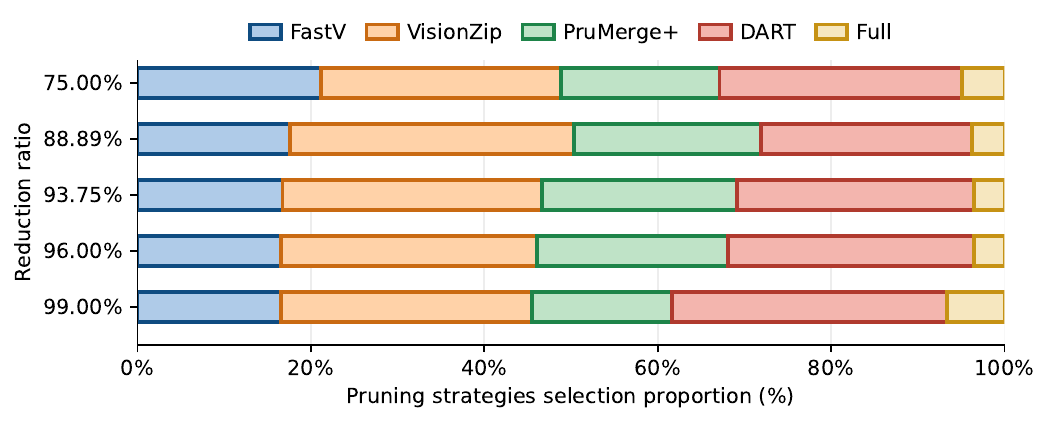}
    \caption{
   VIP-Router strategy-selection distribution across visual-token reduction
ratios.
Each bar reports the mean fraction of test samples routed to FastV,
VisionZip, PruMerge\textsuperscript{+}, DART, or Full over five random seeds.
    }
    \label{fig:strategy_selection_distribution}
\end{figure*}

\subsubsection{VIP-Router Selection Behavior}
\label{app:routing_behavior}

Figure~\ref{fig:strategy_selection_distribution} shows the strategy-selection distribution of VIP-Router across reduction ratios.
The routing decisions remain distributed across all four pruning methods at every operating point, indicating that VIP-Router does not collapse to another globally fixed pruning policy.
The relative preference among pruning strategies also changes with the pruning level; for example, DART is selected increasingly often under aggressive reduction, whereas VisionZip remains a substantial component
throughout the evaluated range.

Full-token inference remains a minority choice.
VIP-Router selects Full for \(4.92\%\), \(3.72\%\), \(3.51\%\), \(3.48\%\), and \(6.66\%\) of samples at reduction ratios of \(75.00\%\), \(88.89\%\), \(93.75\%\), \(96.00\%\), and \(99.00\%\), respectively.
Thus, the accuracy improvement is not obtained by routinely reverting to full-token inference.
Together with the counterfactual bound in Remark~\ref{rem:full_fallback_bound}, these results indicate that the primary gain comes from sample-adaptive selection among heterogeneous pruning strategies.

\subsubsection{Routing Error Diagnosis}
\label{app:routing_diagnosis}

We analyze the saved test predictions of the final VIP-Router across five random seeds, with 5,040 image--question-record--ratio decisions per seed.
Let \(s_i^\star\) denote the oracle argmax returned by the evaluation procedure and \(\hat{s}_i\) the router's selected strategy.
We distinguish an oracle-argmax mismatch,
\(\hat{s}_i \neq s_i^\star\), from the incurred utility regret:
\[
\Delta_i
=
U_i(s_i^\star,r_i)-U_i(\hat{s}_i,r_i).
\]
An argmax mismatch is not a VQA answer error and may incur zero regret when multiple strategies attain the same optimal utility.
For mismatched decisions, we define the oracle margin as
\[
m_i = U_i^{(1)}-U_i^{(2)},
\]
where \(U_i^{(1)}\) and \(U_i^{(2)}\) are the two highest candidate utilities, including ties.
We classify mismatches as near-tie when \(m_i \leq 0.03\) and high-margin otherwise.

\begin{table*}[!htbp]
\centering
\caption{
Routing-error diagnosis of VIP-Router on VTC-Bench Group~A.
Values are means across five random seeds.
Mismatch percentages are fractions of all oracle-argmax mismatches, and regret percentages are fractions of total regret; both pool all test record--ratio pairs within each run rather than macro-averaging across ratios.
Mean regret is computed within each decision group on the raw utility scale.
Neighborhood statistics use \(k=10\).
}
\label{tab:routing_error_diagnosis}
\footnotesize
\setlength{\tabcolsep}{4pt}
\renewcommand{\arraystretch}{1.15}

\begin{tabular*}{\textwidth}{
@{\extracolsep{\fill}}
l c c c c c
@{}
}
\toprule
\textbf{Decision group}
& \shortstack{\textbf{Mismatch}\\\textbf{(\%)}}
& \shortstack{\textbf{Regret}\\\textbf{(\%)}}
& \shortstack{\textbf{Mean}\\\textbf{regret}}
& \shortstack{\textbf{Neighbor}\\\textbf{agreement}}
& \shortstack{\textbf{Neighbor entropy}\\\textbf{(nats)}} \\
\midrule

Oracle-argmax match
& --
& 0
& 0
& 0.436
& 0.885 \\

\midrule

All oracle-argmax mismatches
& 100
& 100
& 0.234
& 0.444
& 0.905 \\

\quad Near-tie
& 59.1
& 36.5
& 0.145
& 0.485
& 0.865 \\

\quad High-margin
& 40.9
& 63.5
& 0.363
& 0.385
& 0.963 \\

\bottomrule
\end{tabular*}
\end{table*}

\paragraph{Strategy disagreement versus utility loss.}
Table~\ref{tab:routing_error_diagnosis} shows that high-margin mismatches account for \(40.9\%\) of strategy disagreements but contribute \(63.5\%\) of total regret.
Their contribution to the remaining utility loss is therefore substantially larger than their frequency alone suggests.
Conversely, approximately \(28.9\%\) of all argmax mismatches incur zero regret, reflecting utility-equivalent choices rather than suboptimal routing.
Near-tie mismatches nevertheless contribute \(36.5\%\) of total regret:
a small gap between the two best candidates does not bound the loss from selecting a lower-utility candidate.
Thus, exact agreement with a single oracle argmax is insufficient to assess routing quality; the utility consequences of a decision must also be considered.

\paragraph{Local structure of routing errors.}
To examine the feature-space neighborhoods of these decisions, we retrieve the ten nearest training examples at the same retained-token ratio using cosine similarity.
The representation consists of the frozen global preview features and ratio inputs,
\[
[\mathbf{t}_{\mathrm{global}};
 \mathbf{v}_{\mathrm{global}};
 r;
 \ln(r+10^{-8})],
\]
with training-set feature standardization followed by L2 normalization.
The learned cross-attention representation is not included.
Neighbor agreement measures the fraction of retrieved examples sharing the test example's oracle-argmax label, while neighbor entropy measures the entropy of the retrieved oracle-label distribution.

High-margin mismatches have lower neighbor agreement and higher neighbor entropy than oracle-matched decisions.
Importantly, the lower-agreement pattern does not hold for argmax mismatches as a whole, indicating that the association is specific to the high-margin subset rather than strategy disagreement in general.
These results associate high-margin routing errors with locally mixed oracle labels in the frozen global preview space.
They characterize the structure of the remaining errors without establishing a representational limitation of the complete learned cross-modal router.

\end{document}